\documentclass[lettersize,journal]{IEEEtran}

\usepackage{cite}

\usepackage{algorithmic}
\usepackage{algorithm}
\usepackage{amsmath,amsfonts}
\usepackage{array}
\usepackage[caption=false,font=normalsize,labelfont=sf,textfont=sf]{subfig}
\usepackage{textcomp}
\usepackage{stfloats}
\usepackage{url}
\usepackage{verbatim}
\usepackage{graphicx}
\def\BibTeX{{\rm B\kern-.05em{\sc i\kern-.025em b}\kern-.08em
    T\kern-.1667em\lower.7ex\hbox{E}\kern-.125emX}}
\usepackage{balance}

\usepackage[utf8]{inputenc} 
\usepackage[T1]{fontenc}    
\usepackage{url}            
\usepackage{booktabs}       
\usepackage{amsfonts}       
\usepackage{nicefrac}       
\usepackage{microtype}      
\usepackage{tabularx}       
\usepackage{enumitem}
\usepackage{mathtools}
\usepackage{amsmath}
\usepackage{bm}
\usepackage{amsthm}
\usepackage{graphicx}
\usepackage{multirow}
\usepackage{courier}
\usepackage[binary-units=true]{siunitx}
\usepackage{makecell}
\usepackage{setspace}
\usepackage{todonotes}
\usepackage{xcolor}
\usepackage{svg}
\usepackage{bbm}
\usepackage{xspace}

\usepackage{pifont}
\DeclareSIUnit\px{px}
\pdfoutput=1

\newtheorem{defn}{Definition}

\newtheorem{thm}{Theorem}

\newtheorem{corol}{Corollary}

\newtheorem{lemma}{Lemma}

\newtheorem*{example*}{Example}
\newtheorem*{prop*}{Proposition}
\newtheorem*{thm*}{Theorem}
\newtheorem*{prv*}{Proof}
\newtheorem*{rmq*}{Remark}

\newcommand{\ie}{\textit{i.e. \xspace}}
\newcommand{\eg}{\textit{e.g. \xspace}}

\newcommand\norm[1]{\left\lVert#1\right\rVert}

\newcommand{\diag}      {\ensuremath{\mathrm{diag}}}

\newcommand{\circulant} {\ensuremath{\mathrm{circ}}}
\newcommand{\toeplitz} {\ensuremath{\mathrm{toep}}}
\newcommand{\bcirc}     {\ensuremath{\mathrm{blkcirc}}}

\newcommand{\cin}{{c_{\mathrm{in}}}}
\newcommand{\cout}{{c_{\mathrm{out}}}}

\def\ddefloop#1{\ifx\ddefloop#1\else\ddef{#1}\expandafter\ddefloop\fi}

\newcommand{\ci}{\ensuremath \mathbf{i}}
\renewcommand{\vector}{\ensuremath{\mathrm{vec}}\xspace}

\def\ddef#1{\expandafter\def\csname #1bb\endcsname{\ensuremath{\mathbb{#1}}}}
\ddefloop ABCDEFGHIJKLMNOPQRSTUVWXYZ\ddefloop

\def\ddef#1{\expandafter\def\csname #1set\endcsname{\ensuremath{\mathcal{#1}}}}
\ddefloop ABCDEFGHIJKLMNOPQRSTUVWXYZ\ddefloop

\def\ddef#1{\expandafter\def\csname #1mat\endcsname{\ensuremath{\mathbf{#1}}}}
\ddefloop ABCDEFGHIJKLMNOPQRSTUVWXYZ\ddefloop

\def\ddef#1{\expandafter\def\csname #1matsf\endcsname{\ensuremath{\mathsf{#1}}}}
\ddefloop ABCDEFGHIJKLMNOPQRSTUVWXYZ\ddefloop

\def\ddef#1{\expandafter\def\csname #1vec\endcsname{\ensuremath{\mathbf{#1}}}}
\ddefloop abcdefghijklmnopqrstuvwxyz\ddefloop

\usepackage[colorlinks,allcolors=blue]{hyperref}

\begin{document}

\title{Spectral Norm of Convolutional Layers\\with Circular and Zero Paddings}

\author{Blaise Delattre, Quentin Barthélemy and Alexandre Allauzen
\thanks{Blaise Delattre and Quentin Barthélemy are with the company FOXSTREAM, Vaulx-en-Velin, France. (e-mail: blaise.delattre@dauphine.eu).}
\thanks{Blaise Delattre and Alexandre Allauzen are with the Miles Team, LAMSADE, Universit\'e Paris-Dauphine, PSL University, Paris, France.}
\thanks{Alexandre Allauzen is with ESPCI PSL, Paris, France.}
}

\maketitle

\begin{abstract}
This paper leverages the use of \emph{Gram iteration} an efficient, deterministic, and differentiable method for computing spectral norm with an upper bound guarantee.
Designed for circular convolutional layers, we generalize the use of the Gram iteration to zero padding convolutional layers and prove its quadratic convergence. We also provide theorems for bridging the gap between circular and zero padding convolution's spectral norm.
We design a \emph{spectral rescaling} that can be used as a competitive $1$-Lipschitz layer that enhances network robustness. Demonstrated through experiments, our method outperforms state-of-the-art techniques in precision, computational cost, and scalability.
\end{abstract}


\section{Introduction}
\IEEEPARstart{C}{onvolutional} neural networks (CNNs) \cite{lecun_gradient_1998} have become pivotal in computer vision and achieve state-of-the-art performance \cite{krizhevsky_imagenet_2012, tan_efficientnetv2_2021}.
The essence of CNNs lies in a cascade of convolutional layers, crucial for extracting hierarchical features.
Convolutional architecture gradually reduces the spatial dimensions of input while increasing the semantic dimension and maintaining the separability of classes \cite{mallat_understanding_2016}.
CNNs are both sparse and structured: indeed the weight-sharing topology of such networks encodes inductive biases, like locality and translation-equivariance. From these biases arises a very powerful architecture for computer vision, in terms of speed, parameter efficiency, and performance.  
Beyond visual tasks, convolutional layers contribute to applications like natural language and speech processing \cite{li_survey_2022, conneau2017very, gu_efficiently_2022,Baevski20Wav2vec} (mainly with one-dimensional convolution), and to graph data \cite{zhang_graph_2019} for applications in biology for instance.

Despite their long-standing use, CNNs are still the focus of lively research on many theoretical and empirical aspects. 
For instance, the relationship between the dataset structure and CNN architectures is explored using spectral theory \cite{pinson_linear_2023}. Examining CNNs from a different perspective, some approaches investigate their properties through the Hessian rank of their matrix representation \cite{singh_hessian_2023}. Padding in convolutional layers is the subject of various studies, analyzing its impact on frequency \cite{tang_defects_2023}, adversarial robustness \cite{gavrikov_interplay_2023}, and encoded position information \cite{islam_how_2019, semih_kayhan_translation_2020}. The effects of stride on convolutional and pooling layers, particularly addressing aliasing issues, are considered and corrected within the framework of sampling theory \cite{zhang_making_2019}. All these contributions showed that CNN layers and their design can still be improved, thanks to theoretical exploration.  

The spectral norm of a matrix plays an important role in understanding the generalization and robustness properties of neural networks.
The work of \cite{bartlett_spectrally-normalized_2017} shows that a bound on the spectral norm of individual layers plays an important role in the generalization bound of neural networks.  Regularizing it enhances generalization and stabilizes training \cite{miyato2018spectral}.
For linear operators, spectral norm coincides with the Lipschitz constant for the $\ell_2$-norm which is a key quantity for adversarial robustness \cite{scaman2018lipschitz};  it guards against perturbation to the input \cite{cisse_parseval_2017, tsuzuku2018lipschitz}, ensuring the robustness and adaptability of CNNs in various deep learning applications. 
However, efficiently estimating spectral norms, especially for convolutional layers, remains challenging. Our work looks to address this by leveraging the \emph{Gram iteration} approach \cite{delattre_efficient_2023} to bound convolutional layer spectral norms.
In this article, we generalize \emph{Gram iteration}  for any matrix norm, we prove \emph{quadratic convergence} to the true spectral norm and we detail the upper bounding property, as detailed in Section~\ref{ssec:spectral_norm_gram}.
We extend the spectral norm computation for convolutional layers with zero padding, see Section~\ref{section:spectral_norm_toep_and_circ}. The resulting estimation serves as a guaranteed upper bound on the spectral norm, crucial for robustness applications.
We explore the computational efficiency of calculating spectral norms for large spatial-sized images and examine the relationship between bounds derived for lower and higher input spatial sizes, as well as the spectral norm bound gap between zeros-padded and circular padding convolutions. Theoretical results are presented in Section~\ref{section:approx}.
We also derived a new $1$-Lipschitz layer, using a \emph{Spectral Rescaling} (SR) to produce Lipschitz neural networks robust to adversarial attacks, see Section~\ref{section:spectral_rescaling}.

\section{Related Work}
Popular methods to compute spectral norm for matrices are \emph{Power iteration} with linear convergence and \emph{Singular values decomposition} (SVD), \cite{golub2000eigenvalue}.
Regarding usage in deep learning, the backward pass of SVD is ill-conditioned \cite{wang_robust_2022}, which makes the method difficult to use for regularization during training. 
Recently, \cite{delattre_efficient_2023} proposed a fast iterative method on GPU called \emph{Gram iteration} which is precise and differentiable.
However, applying those methods directly to the convolutional layer is very costly as the size of the matrix representing the convolutional operator is very large. 
To address this issue, \cite{sedghi2019singular} represents the convolutional layer using a doubly-block circulant matrix. Leveraging the \emph{Fourier transform}, this matrix can be block-diagonalized, and SVD is applied to each block.
Yet, SVD has a large computational cost on GPU making this method quite costly. As computing the singular values is not trivial for zeros-padded convolutions or requires a large computational cost, \cite{wang2020orthogonal} proposes a regularization of singular values by enforcing orthogonalization cheaply. Such orthogonalization has been further studied in \cite{el_mehdi_existence_2022}.

Prior works like \cite{miyato2018spectral, farnia2019generalizable} have adapted the power iteration method to efficiently estimate spectral norms for convolutional layers. However, these estimates do not provide guaranteed upper bounds on the true spectral norm value. Having a certified upper bound is critical for robustness analysis and certification against adversarial perturbations \cite[Proposition 1]{tsuzuku2018lipschitz}. Without a true upper bound, we cannot reliably calculate the worst-case variation of a layer to derive robustness certificates. More precisely, an upper bound on the spectral norm enables formally certifying model robustness, while loose estimates can fail to fully account for the layer's Lipschitz properties.
Recently some approaches have provided methods to derive a scalable upper bound on the spectral norm for convolutional layers, \cite{araujo2021lipschitz, yi2020asymptotic} inspect the spectral density of the matrix representing the convolutional operator to derive bounds and algorithms to compute its spectral norm. \cite{singla2021fantastic} derives an upper bound independent of the input size of the image. However, those methods do not converge to the true spectral norm. The work of \cite{delattre_efficient_2023} uses \emph{Gram iteration} for circular padded convolutional layers to produce an upper bound estimate that converges to the spectral norm. 

Bounding the spectral norm of a layer is a central tool to control the Lipschitz constant of a neural network.
A function $f: \mathbb{R}^d \to \mathbb{R}^m$ is $L$-Lipschitz for the $\ell_2$ norm iff $\norm{f(x) - f(y)}_2 \leq L \norm{x - y}_2$ for all $x, y \in \mathbb{R}^{d}$. The ratio of the Lipschitz constant of the neural network and the input's margins can be used to certify robustness \cite{tsuzuku2018lipschitz} and is called the certified radius.
Spectral norm computations \cite{miyato2018spectral, farnia2019generalizable} can be used to perform \emph{spectral normalization} (SN) on the network layers, it produces robust $1$-Lipschitz layers that result in robust neural network architectures. Indeed, composing $1$-Lipschitz layers gives a $1$-Lipschitz neural network.

A downside of SN is that it can shrink gradient magnitudes, harming training stability and limiting model scalability and performance. 
To mitigate this, techniques like orthogonalization have emerged to construct gradient-preserving layers \cite{li2019preventing, trockman2021orthogonalizing, skew2021sahil}. However, full orthogonalization has substantial computational costs for convolutional layers that limit scaling.
The recently proposed \emph{almost-orthogonal Lipschitz} (AOL) approach \cite{avidan_almost-orthogonal_2022} offers a middle ground between SN and full orthogonalization. Through efficient channel-wise rescaling, AOL layers provide improved gradient propagation compared to spectral normalization, while having minimal extra computational overhead. 
Recent work has explored more advanced composite residual layer architectures, such as \emph{Convex Potential Layer} (CPL) \cite{meunier_dynamical_2022} which draws inspiration from dynamical systems, and \emph{SDP-based Lipschitz Layer} (SLL) \cite{araujo_unified_2022} which extends its frameworks with techniques for Lipschitz constant estimation \cite{fazlyab2019efficient}. These proposed residual layers help gradient propagation and incorporate SN or AOL to regulate layer stability. 
However, a limitation of the use of AOL is that it fails to tightly bound the spectral norm of layer weights, leading to over-constrained rescaled layers. This over-regularization can result in decreased robustness guarantees.

\section{Background on convolutional layers}
\label{section:circulant_and_toeplitz_matrices}

\subsection{Notations}
We define the operator norm of matrix $A$ as 
$\|A\|_{p} = \max_{x \neq 0} \frac{\|Ax\|_{p}}{\|x\|_{p}}$.
For $p=2$, the spectral norm $\norm{A}_2$ equals the largest singular value $\sigma_1(A)$, and for $p=\infty$, $\|A\|_{\infty} = \max_i \sum_j |a_{ij}|$. Frobenius norm of $A$ is denoted $\norm{A}_F$.
We say that a matrix norm $\norm{.}$ is consistent, if for two matrices $A, B$, $\norm{AB} \leq \norm{A} \norm{B}$.

$A^*$ corresponds to the conjugate transpose matrix of $A$.
We define the Kronecker product of two matrices $A$ and $B$ by $A \otimes B$ and the convolution product by $A \star B$.

We denote $\circulant(\cdot)$ (respectively $\toeplitz(\cdot)) $ the operator generating a circulant (respectively Toeplitz) matrix from a vector.
The discrete Fourier transform (DFT) matrix $U \in \Cbb^{n \times n}$ is defined such that, for $1 \leq u,v \leq n$,   $U_{u,v} = e^{-\frac{2\pi\ci uv}{n}}$ and its inverse is defined as $U^{-1} = \frac1n \ U^*$.
We designate maximum non-trivial zero padding for convolutions when input $X \in \Rbb^{\cin \times n \times n}$ is padded towards $\Rbb^{\cin \times n  + k-1 \times n + k-1}$ with $k$ the size of the convolutional kernel.

\subsection{Toeplitz convolutional layers}

Convolutional layers are an essential building block of convolutional neural networks.
Let the input be $X \in \Rbb^{\cin \times n \times n}$, and $K \in \Rbb^{\cout \times \cin \times k \times k}$ a convolution filter. 
We suppose that $Y \in \Rbb^{\cout \times n \times n}$ the output of the convolutional layer $K \star X = Y$ 
has the same spatial size as the input $X$ by selecting the correct padding with zeros. We suppose that the stride equals to $1$.
For $1 \leq j \leq \cout$, $1 \leq u, v \leq n$, $Y$ expresses as:
\begin{equation}
\label{eq:convolution_product}
   Y_{j,u,v} = \sum_{i=1}^\cin \sum_{k_1=1}^n \sum_{k_2=1}^n X_{i, u + k_1,v + k_2} K_{j, i, k_1,k_2} \ ,
\end{equation}

where $K_{j,i, k_1,k_2} = 0$ if $k_1, k_2 \notin \{ 1, \dots,k\}$.
 
The convolutional layer can be expressed as a matrix-vector product, where the structure of matrix $T \in \Rbb^{\cout n^2 \times \cin n^2}$ takes into account the weight sharing of convolution \cite{jain1989fundamentals}:
\begin{equation}
  \label{eq:convolution_product_toep}
  \vector(Y) = \vector(K \star X) = T \vector(X) \ ,
\end{equation}
where $T \in \Rbb^{\cout n^2 \times \cin n^2}$ is a $\cout \times  \cin $ block matrix, where each $T_{j,i, ::}$ is a $n^2 \times n^2$ banded doubly Toeplitz matrix formed with kernel $K_{j,i}$. See representation in Appendix~\ref{section:appendix_toeplitz_matrix}.
As described in Section~2.2 of \cite{yi2020asymptotic}, $T^\Uparrow \in \mathbb{R}^{\cout n^2 \times \cin n^2}$ is a multi-level block Toeplitz matrix where entries are $\cout \times \cin$ matrices. To compute the Lipschitz constant of the convolutional layer it involves computing the maximum singular value of $T$: $\sigma_1(T)$.

\subsection{Circular convolutional layers}

When we pad the input $X$ with circular padding \ie we wrap the image resulting in a circular input, which allows decomposition in the Fourier basis to ease spectrum analysis. The convolutional layer output can be expressed as:
\begin{equation}
  \label{eq:convolution_product_circ}
  \vector(Y) = C \vector(X) \ .
\end{equation}
In the same manner, as for zero padding with matrix $T$, $C$ can also be represented as a $\cout \times  \cin $ block matrix of a doubly circulant matrix.
$C_{j,i, ::}$ is fully determined by the kernel $K_{j, i}$. 
See representation in Appendix~\ref{section:appendix_circulant_matrix}.
In the following, we denote $C_{j,i, ::} = \bcirc(K_{j, i, 1, :}, \dots, K_{j, i, n, :})$. 

Now, based on the doubly-block circulant structure of convolution with circular padding, we present some useful properties.
Operating on a vectorized 2D input $X$, the 2D discrete Fourier transform is expressed as $\vector(U^\top X U) = (U \otimes U) \vector(X)$. We denote $F = U \otimes U$ in the following of the paper.

\begin{thm}[Section 5.5 of~\cite{jain1989fundamentals}]
Let $K \in \Rbb^{n \times n}$ a convolutional kernel and $C \in \Rbb^{n^2 \times n^2}$ be the doubly-block circulant matrix such that $C = \bcirc\left(K_1, \dots, K_n\right)$, then, $C$ can be diagonalized as follows:
\begin{equation}
  C = F \diag(\lambda) F^{-1}
\end{equation}
where $\lambda = F \vector(K)$ are the eigenvalues of $C$.
\end{thm}

In the context of deep learning, convolutions are applied with multiple channels: $\cout$ convolutions are applied on input with $\cin$ channels.
As also studied by \cite{sedghi2019singular} and \cite{yi2020asymptotic},
this type of block doubly-block circulant matrix can be \emph{block} diagonalized as follows.

\begin{thm}[Corollary A.1.1. of~\cite{trockman2021orthogonalizing}]
\label{theorem:block_diagonalization}
Let $P_\text{out} \in \Rbb^{\cout n^2 \times \cout n^2}$ and $P_\text{in} \in \Rbb^{\cin n^2 \times \cin n^2}$ be permutation matrices and let $C \in \Rbb^{\cout n ^2 \times \cin n^2}$ be the matrix equivalent of the multi-channel circular convolution defined by filter $K \in \Rbb^{\cout \times \cin \times n \times n}$, then $C$ can be block diagonalized as follows:
\begin{equation}
  C = (I_{\cout} \otimes F) P_\text{out} D P_\text{in}^\top (I_{\cin} \otimes F^{-1})
\end{equation}
where $D$ is a block diagonal matrix with $n^2$ blocks of size $\cout \times \cin$ and where 
\begin{equation}
  D = P_\text{out}^\top 
   \begin{pmatrix}
     F \vector(K_{1, 1}) & \cdots & F \vector(K_{1, \cin}) \\
     \vdots & & \vdots \\
     F \vector(K_{\cout, 1}) & \cdots & F \vector(K_{\cout, \cin})
   \end{pmatrix} 
   P_\text{in} \ .
\end{equation}
\end{thm}
Matrices $P_\text{out}$ and $P_\text{in}$ reshaping matrix $D$ into a block diagonal matrix correspond to an alternative vectorization of layer input \cite{sedghi2019singular,yi2020asymptotic} while keeping singular values identical \cite{henderson1981TheVM}.

Based on this result and the properties of block diagonal matrices, it is easy to compute the largest singular value $\sigma_1$ of $C$ from the block diagonal matrix $D$.
Let $(D_i)_{i = 1, \dots, n^2}$ be the diagonal blocks of the matrix $D$, then:
\begin{equation}
  \label{equation:lip_conv_circu_expression}
  \sigma_1(C) = \sigma_1(D) = \max_i \sigma_1(D_i) \ .
\end{equation}
Computing the spectral norm of $C$ implies computing the $n^2$ spectral norms of $\cout \times \cin$ matrices.

\subsection{Convolutional layers for different kernel sizes}

We distinguish between the multichannel filter padded towards input size $n$  denoted as $K^\Uparrow \in \Rbb^{\cout \times \cin \times n \times n}$, and filter padded towards lower input spatial size $n_0$, with $k << n_0 \leq n$, $K^\uparrow \in \Rbb^{\cout \times \cin \times n_0 \times n_0}$.
Considering the block diagonal matrix $D_i^\Uparrow$ built from filter $K^\Uparrow$, we re-index $(D_i^\Uparrow)_{i = 1, \dots, n^2}$ into $(D_{u,v}^\Uparrow)_{1 \leq u,v \leq n}$, such that $D_{uv}^\Uparrow = D_{u,v}^\Uparrow$, we note $C^\Uparrow$ the associated matrix representing the circular convolutional layer. In the same manner, we define $(D_{u,v}^\uparrow)_{1 \leq u,v \leq n}$ built from filter $K^\uparrow$ and associated $C^\uparrow$.
With $U^\Uparrow \in \mathbb{C}^{n \times n}$ the DFT matrix 
gives an expression of matrix $D_{u,v}^\Uparrow$:
\begin{align*}
  D_{u,v}^\Uparrow &= \sum_{k_1=0}^{n-1} \sum_{k_2=0}^{n-1} U_{k_1, u}^\Uparrow \ K^\Uparrow_{:,:,k_1, k_2} \ U_{k_2, v}^\Uparrow  \\
   &= \sum_{k_1=0}^{k-1} \sum_{k_2=0}^{k-1} e^{-\frac{2\pi\ci k_1 u}{n}} \ e^{-\frac{2\pi\ci k_2 v}{n}} \ K_{:,:,k_1, k_2} \ ,
\end{align*}
as $K^\Uparrow$ is zero-padded.
Same expression can be obtained for $D_{u,v}^\uparrow = \sum_{k_1=0}^{k-1} \sum_{k_2=0}^{k-1} e^{-\frac{2\pi\ci k_1 u}{n_0}} \ e^{-\frac{2\pi\ci k_2 v}{n_0}} \ K_{:,:,k_1, k_2} 
$.

Similarly, we define,  $T^\Uparrow$ and $T^\uparrow$ respectively constructed from  filters $K^\Uparrow$ and $K^\uparrow$.


\section{Gram iteration}
\label{ssec:spectral_norm_gram}

In this section, we revisit \emph{Gram iteration} introduced in \cite{delattre_efficient_2023}, it applies matrix squaring and Frobenius norm to produce a sequence of iterates that converge and upper bound the spectral norm.
We leverage Gelfand's formula \cite{horn_matrix_2012} to generalize the result to any matrix norm and prove quadratic convergence to the spectral norm.
All proofs of the article are deferred in the Appendix.

\begin{thm}
\label{thm:main_result}
Let $\norm{\cdot}$ be a matrix norm.
Let $W \in \mathbb{C}^{p \times q}$ and define the recurrent sequence  $W^{(t+1)} = {W^{(t)}}^* W^{(t)}$, with $W^{(1)} = W$. 
Let $(\norm{W^{(t)}}^{2^{1-t}})$ be another sequence based on $W^{(t)}$.
Then, $\forall t \geq 1$, we have the following results:
\begin{itemize}[parsep=0pt,itemsep=0pt,topsep=0pt,leftmargin=10pt]
    \item The sequence converges to the spectral norm:
    \begin{equation*}
        \norm{W^{(t)}}^{2^{1 -t}} \underset{t \to +\infty}{\longrightarrow} \sigma_1(W) \ ,
    \end{equation*}
    and this convergence is R-quadratic.
    \item If $\norm{\cdot}$ is consistent, the sequence is an upper bound on the spectral norm:
    \begin{equation*}
        \sigma_1(W) \leq  \norm{W^{(t)}}^{2^{1 -t}} \ .
    \end{equation*}
\end{itemize}
\end{thm}

This theorem is implemented in Algo.~\ref{algo:gram_iteration_dense} with a rescaling to avoid overflow.
For that, matrix $W$ is rescaled at each iteration, and scaling factors are cumulated in variable $r$, to unscale the result at the end of the method, see Appendix~A3 of \cite{delattre_efficient_2023}. Unscaling is crucial as it is required to remain a strict upper bound on the spectral norm at each iteration of the method.

\begin{algorithm}[h]
\caption{: Gram\_iteration$(W, N_\text{iter})$}
\label{algo:gram_iteration_dense}
\begin{algorithmic}[1]
  \STATE \textbf{Inputs} matrix: $W$, number of iterations: $N_\text{iter}$
  \STATE  $r \gets 0$ \hfill \text{// initialize rescaling}
  \STATE \textbf{for} $1 \ldots N_\text{iter}$
    \STATE  \quad $r \gets 2(r + \log \norm{W}_F)$ \hfill \text{// cumulate rescaling}
    \STATE \quad $W \gets W/\norm{W}_F$ \hfill \text{// rescale to avoid overflow}
    \STATE \quad $W \gets W^* W$ \hfill \text{// Gram iteration}
  \STATE $\sigma_1 \gets \norm{W}^{2^{-N_\text{iter}}} \exp{(2^{-N_\text{iter}} r})$
  \STATE \textbf{return} $\sigma_1$
\end{algorithmic}
\end{algorithm}

Remark that the first part of the theorem has been mentioned in \cite{friedland1991revisiting}, but for square matrix only, and without practical implementation and experimental validation.


\section{Spectral norm of convolutional layers}
\label{section:spectral_norm_toep_and_circ}


\subsection{Gram iteration to compute the spectral norm of $C$}
Using Equation~(\ref{equation:lip_conv_circu_expression}), one can compute the spectral norm of the convolutional layer represented by $C$ by computing $n^2$ spectral norm of $\cout \times \cin$ matrices. We use the same algorithm as from \cite{delattre_efficient_2023}.
First, a 2D fast Fourier transform (FFT2) is applied on every $2$-d kernel for every $i, j$ input canal and output canal, this operation is denoted as $\mathrm{FFT2}(K)$ in the following algorithm. This operation is quite cheap for reasonably small $n$ as complexity is in $O(n \log_2(n))$.

\begin{algorithm}[h]
\caption{: norm2\_circ$(K, N_\text{iter})$}
\label{algo:gram_iteration_circ}
\begin{algorithmic}[1]
  \STATE \textbf{Inputs} filter: $K$, number of iterations: $N_\text{iter}$
  \STATE $D \gets \mathrm{FFT2}(K) $ \hfill \text{// FFT}
  \STATE $r \gets 0_{n^2}$ \hfill \text{// initialize rescaling}
  \STATE \textbf{for} $1 \ldots N_\text{iter}$
  \STATE \quad \textbf{for} \text{i} \textbf{in} $1 \ldots n^2$ \hfill \text{// for-loop in parallel}
    \STATE  \quad \quad $r_i \gets 2
    (r_i + \log\norm{D_i}_F)$ \hfill \text{// cumulate rescaling}
    \STATE \quad \quad $D_i \gets D_i/\norm{D_i}_F ~ $ \hfill \text{// rescale to avoid overflow}
    \STATE \quad \quad $D_i \gets D_i^* D_i$ \hfill \text{// Gram iteration}
  \STATE  $\sigma_1 \gets \max_i \left\{ \norm{D_i}_F^{2^{-N_\text{iter}}} \exp{(2^{-N_\text{iter}} r_i}) \right\}$
  \STATE \textbf{return} $\sigma_1$
\end{algorithmic}
\end{algorithm}

Computing the gradient of singular values of the matrix can be tricky and prone to be ill-conditioned, for instance with SVD as shown in \cite{song_why_2021, wang_robust_2022}.
For \emph{Gram iteration } the gradient of $\max_i \{ \norm{D_i}_F^{2^{-N_\text{iter}}} \}$ w.r.t $K$ can be computed explicitly in a stable way, as shown in \cite{delattre_efficient_2023}. 
    
\subsection{Gram iteration to compute the spectral norm of $T$}
For a convolutional layer that uses zeros-padding, this operator can be represented as a concatenation of banded doubly block Toeplitz matrix $T$, this $n^2 \cout \times n^2 \cin$ matrix is very large spectral norm computation can not be done directly on it.
The convolutional operator can be used to ease the computation on the large matrix $T$ as used in \cite{wang2020orthogonal} and \cite[Appendix A.1]{avidan_almost-orthogonal_2022}.
Using this and the Gram iteration previously introduced, we present a bound on the spectral norm of $T$.

\begin{thm}
\label{thm:bound_spectral_norm_toeplitz}
For a zeros-padded convolutional layer $T$ parametrized by a filter $K \in \mathbb{R}^{\cout \times \cin \times k \times k}$,
let us note the Gram iterate for matrix $T$ as $T^{(t)}$. 
We define the Gram iterate for filter as $K^{(t+1)}_{i_1, i_2} = \sum_{j=1}^\cout K^{(t)}_{j, i_1} \star K^{(t)}_{j, i_2}$, with convolution done with maximal non trivial padding, with $K^{(1)} = K$.  
\\
For norm $||~\cdot~||_{\infty}$, we have the following bound:
\begin{align*}
    \norm{T^{(t+1)}}_\infty^{2^{-{t}}} 
    \leq  
    \left( \max_{i_2} \sum_{i_1, k^\prime, l^\prime} \left| \sum_{j=1}^\cout   K^{(t)}_{j, i_1} \star K^{(t)}_{j, i_2} \right|_{k^\prime, l^\prime} 
    \right)^{2^{-{t}}} \ .
\end{align*}
For norm $||~\cdot~||_F$, we have: 
\begin{align*}
    \norm{T^{(t+1)}}_F^{2^{-{t}}}
    \leq \left( 
    k^2 \sum_{i_1, i_2, k^\prime, l^\prime} \left| \sum_{j=1}^\cout   K^{(t)}_{j, i_1} \star K^{(t)}_{j, i_2} \right|_{k^\prime, l^\prime}^2 
    \right)^{^{2^{-(t+1)}}} ,
\end{align*}
and $\sigma_1(T) \leq || T^{(t)} ||^{2^{-t}} \underset{t \to \infty}{\longrightarrow} \sigma_1(T)$ with quadratic convergence.
\end{thm}

In practice, the bounding term is quite tight as it corresponds to the case where input is padded to the maximum non-trivial padding case.
Pairing with Equation~(\ref{eq:toeplitz_gram_conv_relation}), we have  Algo.~\ref{algo:gram_iteration_toep} for the choice of matrix norm $||~.~||_{\infty}$.

\begin{algorithm}[h]
\caption{: norm2\_toep$(K, N_\text{iter})$}
\label{algo:gram_iteration_toep}
\begin{algorithmic}[1]
  \STATE \textbf{Inputs} filter: $K$, number of iterations: $N_\text{iter}$
  \STATE $r \gets 0$
  \STATE \textbf{for} $1 \ldots N_\text{iter}$
    \STATE  \quad  $r \gets 2
    (r + \log\norm{K}_F)$
    \STATE \quad  $K \gets K/\norm{K}_F ~ $
    \STATE \quad  $K \gets \mathrm{conv2d}(K, K, \mathrm{padding}=\mathrm{kernel\_size}(K)-1)$
  \STATE  $\sigma_1 \gets \max_i \left\{ \left(\sum_{j, k_1, k_2} |K_{j, i, k_1, k_2}|\right)^{2^{-N_\text{iter}}} \exp{(2^{-N_\text{iter}} r}) \right\}$
  \STATE \textbf{return} $\sigma_1$
\end{algorithmic}
\end{algorithm}


\section{Bound Approximations\\for circular convolutional layers}
\label{section:approx}
For images with large spatial size $n$, computing $n^2$ spectral norm of $\cin \times \cout $ matrices $D_i$ can be quite expensive. We explore how the bound derived for a lower input spatial size $n_0$ is related to the one for a higher input spatial size $n$, see Theorem~\ref{thm:bound_approximation_for_lower_input_size}.
Using similar techniques we also derive a relationship between the spectral norm bound gap between zeros-padded convolution (Toeplitz matrix) and circular padding convolution (circulant matrix) see Theorem~\ref{thm:bound_spectral_norm_toeplitz}.

It this section, ${D_{u,v}^\Uparrow}^{(t)}$ and ${D_{u,v}^\uparrow}^{(t)}$ denote respectively the $t^\text{th}$ iterate of Gram of matrix $D_{u,v}^\Uparrow$ and $D_{u,v}^\uparrow$.

\subsection{Approximation for lower input spatial size $n_0 \leq n$}
\label{section:approx_n0}

A bound on spectral norm can be produced with Gram iteration as in Theorem~\ref{thm:main_result}:
\begin{equation*}
  \sigma_1(C^\Uparrow)
  = \max_{1 \leq u,v\leq n} \sigma_1(D_{u,v}^\Uparrow) 
  \leq \max_{1 \leq u,v \leq n} \norm{{D^\Uparrow}^{(t)}_{u,v}}_F^{2^{1-t}} \ .
\end{equation*}

To consider a very large input spatial size $n$, we can use our bound by approximating the spatial size for $n_0 \leq n$. It means we pad the kernel $K$ to match the spatial dimension $n_0 \times n_0$, instead of $n \times n$. To compensate for the error committed by the sub-sampling approximation, we multiply the bound by a factor $\alpha$. 
The work of \cite{pfister2019bounding} analyzes the quality of approximation depending on $n_0$ and gives an expression for factor $\alpha$ to compensate and ensure to remain an upper-bound, as studied in \cite{araujo2021lipschitz}.

We can estimate the spectral norm bound for a convolution defined by kernel $K$ and input size $n$ with the spectral norm bound for a lower input size $n_0 \leq n$.

\begin{thm} (Bound approximation for lower input size) \\
\label{thm:bound_approximation_for_lower_input_size}
For $n_0 \geq 2^{t} \lfloor \frac{k}{2}\rfloor + 1$, $\alpha = \frac{2^{t} \lfloor \frac{k}{2}\rfloor}{n_0}$, the spectral norm of the circular convolutional layer can be bounded by:
\begin{equation*}
  \sigma_1(C^\Uparrow) \leq \left(\frac{1}{1 - \alpha}\right)^{2^{-t}}  \max_{1 \leq u,v \leq n_0} \norm{{D^\uparrow}^{(t)}_{u,v}}_F^{2^{1-t}} \ .
\end{equation*}
\end{thm}

\begin{figure}
    \centering
    \includegraphics[width=0.8\linewidth]{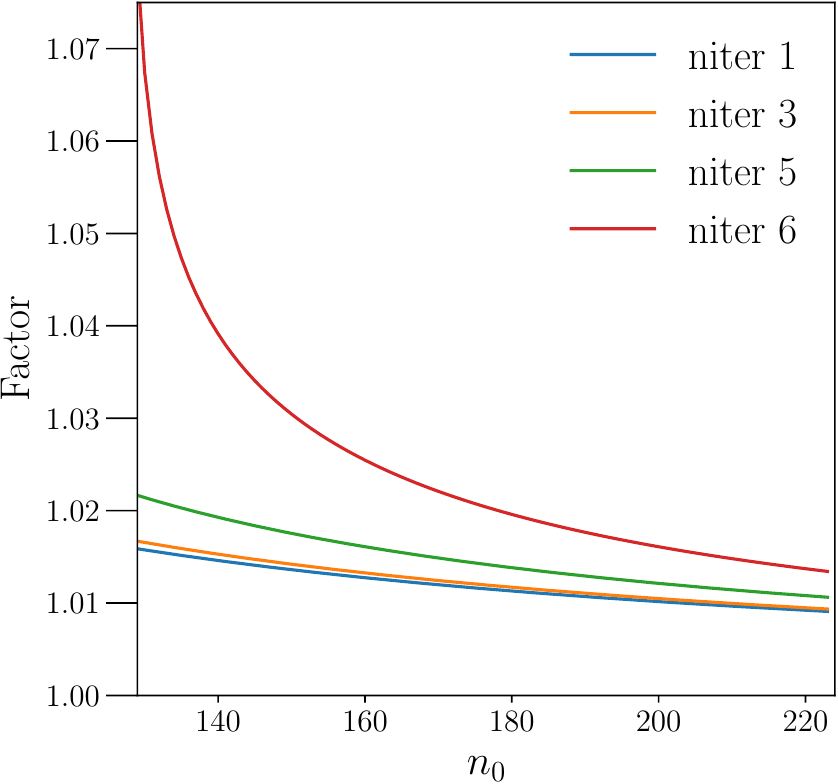}
    \caption{Evolution of bound factor $(1/(1 - \alpha))^{2^{-t}}$, described in Theorem~\ref{thm:bound_approximation_for_lower_input_size}, for input size $n=224$, kernel size $k=3$ and number of Gram iterations  $t \in \{1, 3, 5, 6 \}$. }
    \label{fig:alpha_factor_evolution}
\end{figure}

We illustrate the evolution of the factor term $\left(\frac{1}{1 - \alpha}\right)^{2^{-t}}$ in Figure~\ref{fig:alpha_factor_evolution} for different $n_0$ and for different number of Gram iteration $t$.  We see that the more the number of iterations $t$ increases the larger the factor, however, the estimation of the spectral norm is also much more accurate.
For a very large image size $n$, this theorem allows the computation of a tight bound on the spectral norm of a convolutional layer at a low computational cost w.r.t input spatial size.


\subsection{Approximation between circular and zero paddings}
\label{section:approx_circ}

In CNN architectures, zero-padded convolutions are more common than circular ones. We want to use the circular convolution theory to obtain a bound on the spectral norm of zero padding convolution.

\begin{thm}
\label{thm:bound_circ_toep} (Bound approximation for zero padding) \\
For $n \geq 2^{t} \lfloor \frac{k}{2}\rfloor + 1$, $\alpha = \frac{2^{t} \lfloor \frac{k}{2}\rfloor}{n}$, the spectral norm of the zero-padded convolutional layer can be bounded by:
\begin{equation*}
     \sigma_1(T^\Uparrow) \leq \left(\frac{1}{1 - \alpha}\right)^{2^{-t}} \max_{1 \leq u,v \leq n} \norm{{D^\Uparrow}^{(t)}_{u,v}}_F^{2^{1-t}} \ .
\end{equation*}
\end{thm}

This result can be used when computing the bound on the spectral norm of convolution layers with circular padding is more convenient than zero padding, namely because circulant matrix theory allows the use of Fourier transform to ease computation. It can be crucial in the case of large kernel size $k$.


\section{Spectrally Rescaled Layers}
\label{section:spectral_rescaling}

\subsection{Spectral Rescaling (SR)}
The work \cite{miyato2018spectral} introduces spectral normalization (SN) for layers, dividing all singular values by the spectral norm, \ie largest singular value. Squashing them towards 0, SN generates an ill-conditioned matrix.
The work \cite{avidan_almost-orthogonal_2022} introduces an almost-orthogonal Lipschitz (AOL) rescaling. It is expressed as a diagonal matrix $R_{ii} = \left(\sum_{j=1}^q | W^\top W|_{i,j} \right)^{-1/2}$ such that spectral norm of $WR$ is bounded by 1.
The matter is that the spectral norm of the rescaled matrix $\norm{WR}_2$ can be way lower than $1$. Whereas with spectral normalization the spectral norm is exactly one, but the condition number of the resulting matrix can be very high which is not the case with the rescaling AOL.

We introduce a tight spectral rescaling (SR) that converges to a spectral norm of $1$ and provides a good condition number in contrast to spectral normalization.
We design this rescaling by generalizing $W^\top W$ to the $t$-th iterate of Gram of the matrix $W$.

\begin{thm}
\label{thm:improve_aol}
    For any $W \in \mathbb{R}^{p \times q}$, integer $t \geq 1$, define $W^{(t+1)} = {W^{(t)}}^\top W^{(t)}$, with $W^{(0)} = W$. We define the spectral rescaling
    $R^{(t)}$ as the diagonal matrix with 
    $R^{(t)}_{ii} = \left( \sum_j \left| W^{(t+1)} \right|_{ij} \right)^{-2^{-(t+1)}}$ if the expression in the brackets is non zero, or $R_{ii} = 0 $ otherwise.
    
    Then  $\sigma_1(W R^{(t)}) \underset{t \to +\infty}{\longrightarrow} 1$, with quadratic convergence  and the iterates upper bounds the limit $\sigma_1(W R^{(t)}) \leq 1$.
\end{thm}

\subsection{$1$-Lipschitz layer using Spectral Rescaling}

Using previous Theorem~\ref{thm:improve_aol} and Theorem~1 of \cite{araujo_unified_2022}, we can design the following two $1$-Lipschitz layers.
\begin{corol} 
\label{corol:layers_SR}
    We use notations of previous Theorem~\ref{thm:improve_aol}.
    Let a diagonal matrix $Q = \mathrm{diag}(q_i)$ with $q_i > 0$. 
    The spectral rescaling $R^{(t)}$ is defined as 
    $R^{(t)}_{ii} = \left( \sum_j \left| W^{(t)} \right|_{ij} \frac{q_i}{q_j}\right)^{-2^{-t}}$ if the expression in the brackets is non zero, or $R_{ii} = 0 $ otherwise.
    Then, the following rescaled layers are $1$-Lipschitz:
    \begin{itemize}
    \item the dense layer $x \mapsto WR^{(t)} x + b$ ,
    \item the residual layer $x \mapsto x - 2W {R^{(t)}}^{2} \phi(W^\top x + b)$ ,
    \end{itemize}
    with $b \in \mathbb{R}^q$ a biais, and $\phi$ a $1$-Lipschitz activation (\eg $\mathrm{ReLU}, \mathrm{tanh}, \mathrm{sigmoid}, \dots$).
\end{corol}

Spectral rescaling is differentiable so training is possible through the rescaled layer as a whole. 
Previous Corollary~\ref{corol:layers_SR} can be applied to convolutional layers using Algorithm~\ref{algo:gram_iteration_toep} to compute the rescaling.

Note that for $t=1$, SR is equivalent to AOL. As $t \to \infty$ the SR is more similar to SN,
indeed the stable rank $\|W^{(t)}\|_F^2 / \|W^{(t)}\|_2^2$ of the Gram iterate matrix $W^{(t)}$,
is closer and closer to one.
The Gram iterate for large iteration number $t$ is numerically a rank one matrix and the rescaling is equivalent to spectral normalization. Therefore, our method is an interpolation between the AOL rescaling and the SN rescaling, the number of iterations $t$ is the cursor which allows us to select at which degree the rescaling is more AOL or SN. 
Note that SN used on the SLL block of \cite{araujo_unified_2022} retrieves the CPL block of \cite{meunier_dynamical_2022}.

One main advantage of our method over SN done with power iteration is that it provides a  guarantee that the layer is $1$-Lipschitz as the rescaling is upper bounding the singular values. We see empirically that power iteration (with $100$ iterations) can be a loose lower bound on spectral norm as depicted in Figure~\ref{fig:accuracy_time_conv_spectral_norm}. Moreover, SR is deterministic and converges quadratically, whereas SN with power iteration is random and converges linearly.



\section{Numerical experiments}
\label{section:experiments}

For research reproducibility, the code of experiments is available at
\url{https://github.com/blaisedelattre/lip4conv}.

\begin{figure}[h]
    \centering
    \includegraphics[width=0.95\linewidth]{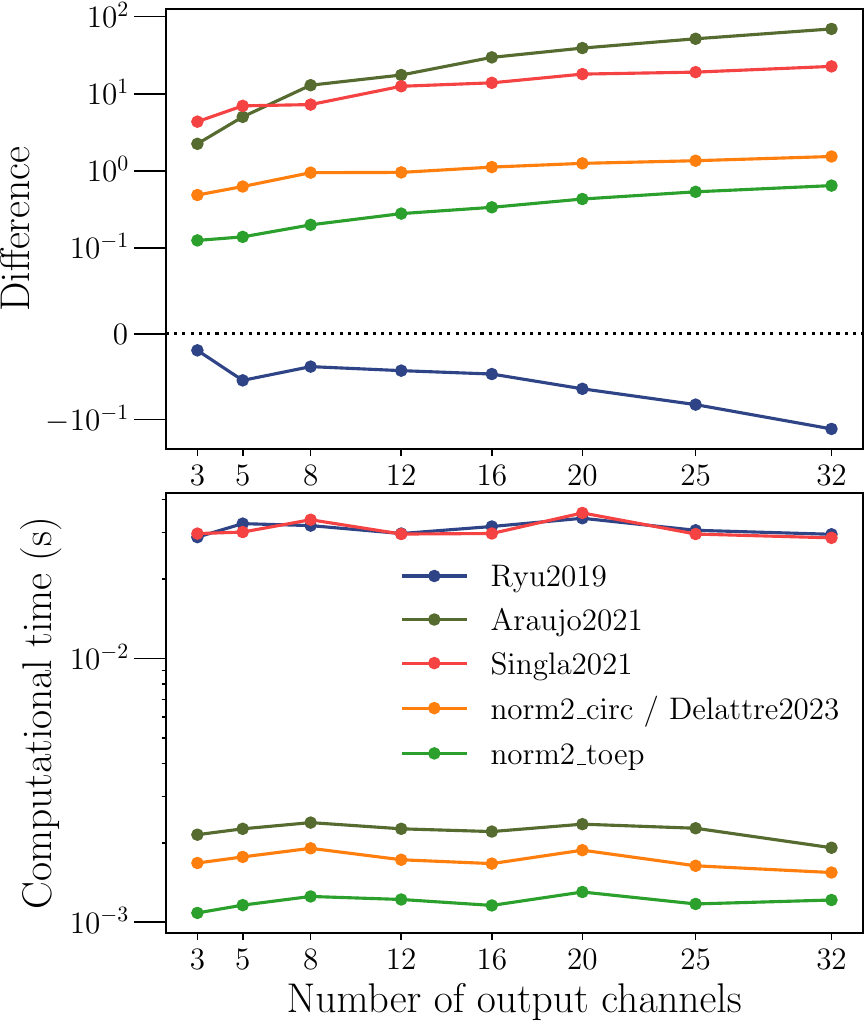}
    \caption{Estimation difference and computational times for spectral norm computation for zeros padded convolutional layers varying number of channels $\cin, \cout$, comparing different methods.
    Kernel size is $3$, input size is $32$.}
    \label{fig:accuracy_time_conv_spectral_norm}
\end{figure}

\subsection{Estimation of spectral norm of convolutional layers}
We assess the accuracy in the estimation of the spectral norm of zeros padded convolutional layers and the computational cost of state-of-the-art methods. The case of circular padded convolutional layer has been studied in \cite{delattre_efficient_2023}, however, they could not provide theoretical guarantees of upper bound for zeros padded convolutional layer.
We do not compare with the \cite{sedghi2019singular} method because \cite{delattre_efficient_2023} here $\mathrm{norm2\_toep}$, produces the same reference values for the spectral norm of the circular padding convolutional in a much shorter computational time. For ground truth, we generate random filters $K$ from Gaussian distribution and generate the associated concatenation of doubly banded Toeplitz matrix, as described in Section~\ref{section:circulant_and_toeplitz_matrices}, we then compute its spectral norm using NumPy $\mathrm{matrix\_norm}$. We use $6$ iterations for Gram iteration-based methods, and $100$ for power iteration \cite{Ryu19Plug}. We only reported for $1 \leq \cin, \cout \leq 32$ because otherwise the Toeplitz matrix would be too large and could not fit into memory. The $\mathrm{norm2\_circ}$ method is the circulant approach with the use of corrective factor $\alpha$ to ensure upper-bound property using Theorem~\ref{thm:bound_circ_toep}. 

The power iteration method has no guarantee to provide an upper bound on the spectral norm and it can be observed in the following Figure~\ref{fig:accuracy_time_conv_spectral_norm}:
Ryu2019's method is lower bounding the true spectral norm. Our method $\mathrm{norm\_toep}$ provides the best accuracies with low computational cost.


\begin{table}[h]
\caption{Accuracies in $\%$ for SLL model on CIFAR-10, for different rescalings: AOL and SR.}
\label{tab:evaluation_results_cifar10}
\centering
\begin{tabular}{|l|c|c|c|c|c|}
    \hline
    \textbf{Rescaling} & \textbf{Accuracy} & \multicolumn{4}{c|}{\textbf{Certified accuracy ($\varepsilon$)}} \\
    \cline{3-6}
    & & $0.141$ & $0.283$ & $0.423$ & $1$  \\
    \hline
    \text{AOL (S)} & 71.06 & 62.78 & 53.67 & 45.37 & 19.18  \\
    \text{SR (S)} & \textbf{72.44} & \textbf{63.49} & \textbf{54.66} & \textbf{46.01} & \textbf{19.62} \\
    \hline
    \text{AOL (M)} & 72.41 & 63.72 & 54.48 & 46.38 & 19.92  \\
    \text{SR (M)} & \textbf{73.38} & \textbf{64.64} & \textbf{55.43} & \textbf{46.79} & \textbf{20.18} \\
    \hline
\end{tabular}
\end{table}

\begin{table}[h]
\caption{Accuracies in $\%$ for SLL model on CIFAR-100, for different rescalings: AOL and SR.}
\label{tab:evaluation_results_cifar100}
\centering
\begin{tabular}{|l|c|c|c|c|c|}
    \hline
    \textbf{Rescaling} & \textbf{Accuracy} & \multicolumn{4}{c|}{\textbf{Certified accuracy ($\varepsilon$)}} \\
    \cline{3-6}
    & & $0.141$ & $0.283$ & $0.423$ & $1$  \\
    \hline
    \text{AOL (S) } & 46.05 & 35.13 & 26.72 & 20.49 & 0.073  \\
    \text{SR (S)} & \textbf{46.62} & \textbf{35.46} & \textbf{27.42} & \textbf{21.14} & \textbf{0.075}  \\
    \hline
    \text{AOL (M)} &  46.50 & 35.81 & 27.39 & 21.21 & 7.78  \\
    \text{SR (M)} & \textbf{47.50} & \textbf{36.50} & \textbf{28.54} & \textbf{21.89} & \textbf{8.13} \\
    \hline
\end{tabular}
\end{table}

\subsection{Certified robustness with spectrally rescaled layers}
This experiment demonstrates the impact of our new SR method for designing  1-Lipschitz layers to produce robust networks. Indeed, if each layer in the network is $1$-Lipschitz then the overall network is also $1$-Lipschitz. 

We evaluate our method on CIFAR-10 and CIFAR-100 datasets for a certified classification task and compare it with the original SLL block which uses AOL rescaling. Our method uses  Algo.~\ref{algo:gram_iteration_toep} with 3 iterations.
We take the same setting of training as \cite{araujo_unified_2022}, i.e. same architecture, loss, and optimizer. We run the experiment for the small (S) architecture which is a network of $20$ SLL convolutional layers followed by $7$ SLL dense layers, and medium (M) architecture $30$ SLL convolutional layers followed by $10$ SLL dense layers.

We report the natural accuracy and the certified accuracy for different levels of perturbation $\epsilon$ in Tables \ref{tab:evaluation_results_cifar10} and \ref{tab:evaluation_results_cifar100}, runs were averaged $3$ times.
Our rescaling provides better overall performance for different levels of perturbation $\epsilon$. We see that our method provides good scaling across the different sizes of architecture (small and medium) and the different datasets (CIFAR-10 and CIFAR-100).
Moreover, our bound is a strict upper bound to the spectral norm so the layers and the overall network can be certifiably Lipschitz-bounded.


%

\section{Conclusion}
In this work, we leveraged the \emph{Gram iteration} algorithm to produce three main contributions.
We improved its theoretical foundation by generalizing it to any matrix norm and proving its quadratic convergence.
We extended the \emph{Gram iteration} to encompass zero padding convolutions, achieving state-of-the-art performance accuracy with practical algorithms while minimizing computational costs. The computed estimate is a differentiable deterministic upper bound on the spectral norm and converges towards it.
Additionally, we've introduced new theorems to bound the gap between spectral norms under zero and circular padding, as well as establishing a theorem linking input size and the bound on convolutional layers. 
Moreover, we devised the \emph{Spectral Rescaling} technique which is used to produce a robust 1-Lipschitz layer enhancing the robustness of the convolutional layers in deep neural networks.

As a prospect for further research, delving into the adaptability of the Gram iteration for computing several singular values, beyond just the spectral norm, presents an intriguing avenue.

\appendix
\subsection{Toeplitz matrix}
\label{section:appendix_toeplitz_matrix}

Let $T \in \Rbb^{\cout n^2 \times \cin n^2}$ is a $\cout \times  \cin $ be a block matrix:
\begin{align*}
T = \begin{psmallmatrix}
    T_{1,1, ::} & T_{1,2, ::} & \cdots & T_{1,\cin, ::} \\
    T_{2,1, ::} & T_{2,2, ::} & \cdots & T_{2,\cin, ::} \\
    \vdots & \vdots & \ddots & \vdots \\
    T_{\cout, 1, ::} & T_{\cout, 2, ::} & \cdots & T_{\cout,\cin, ::}
\end{psmallmatrix} \ ,
\end{align*}
where each $T_{j,i, ::}$ is a $n^2 \times n^2$ banded doubly Toeplitz matrix formed with kernel $K_{j,i}$. The matrix $T_{j,i, ::}$ is represented in Equation~(\ref{eq:mat_tij}), where we define:
\begin{align*}
  & \mathrm{toep}(K_{j, i, k_1}) = \\   
& \begin{psmallmatrix} 
    K_{j, i, k_1, \lfloor\frac{k}{2}\rfloor} & \cdots  & K_{j, i, k_1, 1} & 0 & \cdots & 0\\
    \vdots & K_{j, i, k_1, \lfloor\frac{k}{2}\rfloor} & \ddots & \ddots & \ddots &  \vdots \\
    K_{j, i, k_1, k} & \ddots &\ddots &\ddots &\ddots & 0 \\
    0 & \ddots &\ddots &\ddots &\ddots & K_{j, i, k_1, 1} \\
    \vdots & \ddots &\ddots &\ddots & K_{j, i, k_1, \lfloor\frac{k}{2}\rfloor} & \vdots \\
    0 & \cdots & 0  & K_{j, i, k_1, k} & \cdots  & K_{j, i, k_1, \lfloor\frac{k}{2}\rfloor}
\end{psmallmatrix}
\end{align*}

\begin{figure*}[t]
\begin{align}
\label{eq:mat_tij}
T_{j,i, ::} = 
\begin{psmallmatrix} 
    \mathrm{toep}(K_{j, i, \lfloor\frac{k}{2}\rfloor}) & \cdots  & \mathrm{toep}(K_{j, i, 1}) & 0 & \cdots & 0\\
    \vdots & \mathrm{toep}(K_{j, i, \lfloor\frac{k}{2} \rfloor + 1  }) & \ddots & \ddots & \ddots &  \vdots \\
    \mathrm{toep}(K_{j, i, k}) & \ddots &\ddots &\ddots &\ddots & 0 \\
    0 & \ddots &\ddots &\ddots &\ddots & \mathrm{toep}(K_{j, i, 1}) \\
    \vdots & \ddots &\ddots &\ddots & \mathrm{toep}(K_{j, i, \lfloor\frac{k}{2}\rfloor}) & \vdots \\
    0 & \cdots & 0  & \mathrm{toep}(K_{j, i, k}) & \cdots  & \mathrm{toep}(K_{j, i, \lfloor\frac{k}{2}\rfloor})
\end{psmallmatrix} 
\end{align}
\end{figure*}

\subsection{Circulant matrix}
\label{section:appendix_circulant_matrix}

In the same manner, as for zero padding with matrix $T$, $C$ can also be represented as a $\cout \times  \cin $ block matrix of doubly circulant matrix.
$C_{j,i, ::}$ is fully determined by the kernel $K$ as:
\begin{align*}
C_{j,i, ::} = 
   \begin{psmallmatrix}
    \circulant(K_{j, i, 1, :}) & \circulant(K_{j, i, 2, :}) & \cdots & \circulant(K_{j, i, n, :}) \\
    \circulant(K_{j, i, n, :}) & \circulant(K_{j, i, 1, :}) & \ddots & \vdots \\
    \vdots & \ddots & \circulant(K_{j, i, 1, :}) & \circulant(K_{j, i, 2, :}) \\ 
    \circulant(K_{j, i, 2, :}) & \cdots & \circulant(K_{j, i, n, :}) & \circulant(K_{j, i, 1, :}) \phantom{\vdots}
  \end{psmallmatrix} \ .
\end{align*}
We denote $C_{j,i, ::} = \bcirc(K_{j, i, 1, :}, \dots, K_{j, i, n, :})$.


\subsection{Proof of Theorem~\ref{thm:main_result}}

\begin{defn}
    Suppose that sequence $(x_k)$ converges to limit $L$. The sequence is said to converge Q-quadratically if, for some $\mu > 0$:
    \begin{align*}
        \lim_{k \to \infty} \frac{|x_{k+1} - L |}{|x_{k} - L |^2} = \mu \ .
    \end{align*}
    Moreover, the sequence is said to converge R-quadratically if there exists a sequence
    $(\epsilon_k)$ such that 
    \begin{align*}
        | x_k - L | \leq \epsilon_k, \forall k
    \end{align*}
    and $(\epsilon_k)$ converges Q-quadratically to zero.
\end{defn}

\begin{lemma}
\label{lemma:quadratic_convergence}
For a matrix $W \in \mathbb{C}^{p \times q}$,
convergence is R-quadratic for sequence $\norm{W^{(t)}}^{2^{-(t+1)}}$ for any matrix norm $|| \cdot ||$ .
\end{lemma}

\begin{proof}
As all norms are equivalent in finite dimension there exists $\beta_1, \beta_2 > 0$ such that:
\begin{align*}
    \frac{1}{\beta_1} \norm{W} \leq \norm{W}_2 \leq \frac{1}{\beta_2} \norm{W} \ .
\end{align*}
Applying previous inequality to $W^{(t)} = (W^* W)^{2^t}$ for $t > 1$, and taking the power $2^{-(t+1)}$
\begin{align*}
        \left(\frac{1}{\beta_1} \right)^{2^{-(t+1)}} &\norm{(W^* W)^{2^t}}^{2^{-(t+1)}} 
        \leq \norm{W}_2 \\ 
        &\leq \left(\frac{1}{\beta_2} \right)^{2^{-(t+1)}} \norm{(W^* W)^{2^t}}^{2^{-(t+1)}} \ ,
\end{align*}
where we use that $\norm{(W^* W)^{2^t}}_2^{2^{-(t+1)}} = \norm{W}_2$.

We define sequence $x_t = \norm{(W^* W)^{2^t}}^{2^{-(t+1)}}$ and its limit $L$, $\alpha_t = \left(\frac{1}{\beta_1}\right)^{2^{-(t+1)}}$  we have that 
$\left| x_t - L \right| \leq \epsilon_t$, with $\epsilon_t = L ( \frac{1}{\alpha_t} - 1)$.
Then,
\begin{align*}
    \frac{\epsilon_{t+1}}{\epsilon_{t}^2}
    &= \frac{\frac{1}{\alpha_{t+1}} - 1}{\left(\frac{1}{\alpha_t} - 1\right)^2}
    = \frac{\beta_1^{2^{-(t+2)}} - 1}{\left(\beta_1^{2^{-(t+1)}} - 1\right)^2} \\
    &= \frac{\beta_1^{2^{-(t+2)}} - 1}{\beta_1^{2^{-t}} - 2 \beta_1^{2^{-(t+1)}} + 1} = \frac{\beta_1^{2^{-(t+2)}} - 1}{\beta_1^{2^{-t}} - 1 - 2 (\beta_1^{2^{-(t+1)}} - 1)}\\
    &\sim_{t \to \infty} \frac{2^{-(t+2)} \ln(\beta_1)}{2^{-t} \ln(\beta_1) (1 + 2^{-1})} 
    \sim_{t \to \infty}  \frac{1/4 }{ 3/2} = 1/6 \ .
\end{align*}
\end{proof}


\begin{proof}
    Using Gelfand's formula to the sequence $W^{(t)} = (W^* W)^{2^t}$ for $t > 1$, the result is direct for the convergence.
    Gelfand's formlula also tells us that if 
    if the matrix norm is consistent it is an upper bound on the spectral norm.
    The proof for quadratic convergence comes from Lemma~\ref{lemma:quadratic_convergence}.
\end{proof}


\subsection{Proof of Theorem~\ref{thm:bound_spectral_norm_toeplitz}}


\begin{proof}
\begin{align*}
    T_{j, i, k_2 l_2, k_1 l_1} = K_{i, j, k_1 - k_2 + k -1, l_1 - l_2 + k -1} 
\end{align*}

For $T^{(1)} \in \Rbb^{n^2 \cin \times n^2 \cin}$ the first Gram iterate, $1 \leq i_1, i_2 \leq \cin$,
\begin{align*}
    T^{(1)}_{i_1, i_2, k_1 l_1, k_2 l_2} = \left( \sum_{j=1}^\cout   K_{j, i_1} \star K_{j, i_2} \right)_{k_2 - k_1, l_2 -l_1}
\end{align*}
It was first derived in \cite{avidan_almost-orthogonal_2022}.

We define the Gram iterate for filter $K$, $K^{(t+1)}_{i_1, i_2} = \sum_{j=1}^\cout   K^{(t)}_{j, i_1} \star K^{(t)}_{j, i_2}$, convolution is defined with maximal non trivial padding. We pad $K^{(t)}$ by zeros with the current spatial size of the kernel of $K^{(t)}$ minus one.

We  can extend the previous result for $t$-th iterate of Gram:
\begin{align}
\label{eq:toeplitz_gram_conv_relation}
    T^{(t+1)}_{i_1, i_2, k_1 l_1, k_2 l_2} = \left( \sum_{j=1}^\cout   K^{(t)}_{j, i_1} \star K^{(t)}_{j, i_2} \right)_{k_2 - k_1, l_2 -l_1} \ .
\end{align}
For norm $||~.~||_{\infty}$, we use the tight bound which simplifies to:
\begin{align*}
    \norm{T^{(t+1)}}_\infty &= \max_{i_2, k_2, l_2} \sum_{i_1, k_1, l_1} \left| \sum_{j=1}^\cout   K^{(t)}_{j, i_1} \star K^{(t)}_{j, i_2} \right|_{k_2 - k_1, l_2 -l_1} \\
    &\leq  \max_{i_2} \sum_{i_1, k^\prime, l^\prime} \left| \sum_{j=1}^\cout   K^{(t)}_{j, i_1} \star K^{(t)}_{j, i_2} \right|_{k^\prime, l^\prime} \ .
\end{align*}
For norm $||~.~||_F$, it gives, 
\begin{align*}
    \norm{T^{(t+1)}}_F^2 &=  \sum_{i_1, i_2, k_1, k_2, l_1, l_2} \left| \sum_{j=1}^\cout   K^{(t)}_{j, i_1} \star K^{(t)}_{j, i_2} \right|_{k_2 - k_1, l_2 -l_1}^2  \\
    &\leq k^2 \sum_{i_1, i_2, k^\prime, l^\prime} \left| \sum_{j=1}^\cout   K^{(t)}_{j, i_1} \star K^{(t)}_{j, i_2} \right|_{k^\prime, l^\prime}^2 \ .
\end{align*}
\end{proof}


\subsection{Proof of Theorem~\ref{thm:bound_approximation_for_lower_input_size}}

\begin{thm} [Adapted from Theorem~1 of \cite{pfister2019bounding}]
\label{theorem:approx_trigonometric_polynomial} 
Let $\gamma : \mathbb{R}^2 \rightarrow \mathbb{C}$ be a trigonometric polynomial of degree $d = \left\lfloor k/2 \right \rfloor$ defined by coefficients $\Gamma \in \mathbb{C}^{k \times k}$:
\begin{equation*}
    \gamma(w_1, w_2) = \sum_{k_1=0}^{2d} \sum_{k_2=0}^{2d} 
    \Gamma_{k_1, k_2} \ e^{\ci(w_1 (k_1-d))} e^{\ci(w_2 (k_2-d))} \ .
\end{equation*}
Let $\Omega_n$ be the set of $n$ equidistant sampling points on $[0, 2\pi ]$: $\Omega_n = \left \{\omega_l = \frac{2\pi l}{n} ~ |~ l \in \{1, \dots, n \} \right \}$. \\
Then, for $\alpha = \frac{2d}{n}$, we have:
\begin{equation*}
    \max_{w_1, w_2 \in [0, 2\pi ]^2} \left| \gamma(w_1, w_2) \right| 
    \leq (1 - \alpha)^{-1} \max_{w_1^{\prime}, w_2^{\prime} \in \Omega_n^2} \left| \gamma(w_1^{\prime}, w_2^{\prime}) \right| \ .
\end{equation*}
\end{thm}

We define $E \in \mathbb{C}^{\cout \times \cin }$ the spectral density matrix of $T^\Uparrow$: 
\begin{equation}
    \label{eq:spectral_density_matrix}
    E (w_1, w_2) = \sum_{k_1=0}^{k-1} \sum_{k_2=0}^{k-1} e^{-\ci k_1 w_1} \ e^{-\ci k_2 w_2} \ K_{:,:,k_1, k_2} \ .
\end{equation}
The spectral density matrix $E$ corresponds to the Discrete Time Fourier Transform (DTFT) of $K$. $E^{(t)}$ note  the $t^\text{th}$ iterate of Gram of matrix $E$.

Using Theorem~\ref{theorem:approx_trigonometric_polynomial}, we can bound the Gram iterate maximum norm over $[0, 2\pi]$ of $E$ with the maximum taken over uniform sampled points.

\begin{lemma} (Inequality between the maximum of spectral norm density and density uniformly sampled for Gram iterates) \\
\label{lemma:inequality_btw_max_spectral_norm_density}
For $n_0 \geq 2^{t}  \lfloor \frac{k}{2}\rfloor + 1$ sampling points, $\alpha = \frac{2^{t} \lfloor \frac{k}{2}\rfloor}{n_0}$:
\begin{equation*}
    \max_{w_1,w_2 \in [0, 2\pi]} \norm{E^{(t)}(w_1, w_2)}_F^2  \leq (1 - \alpha)^{-1} \max_{1 \leq u,v \leq n_0} \norm{{D^\uparrow}^{(t)}_{u,v}}_F^{2}
\end{equation*}
\end{lemma}

\begin{proof}
We define  the trigonometric polynomial of degree $2^t \lfloor \frac{k}{2}\rfloor$: $P = \norm{E^{(t)}(~.~, ~.~)}^{2}_F$.
\begin{align*}
    &\max_{w_1,w_2 \in [0, 2\pi]} \norm{E^{(t)}(w_1, w_2)}_F^2 \\
    & = \max_{w_1,w_2 \in [0, 2\pi]^2} P(w_1, w_2) \\
    &\leq  (1 - \alpha)^{-1} \max_{w_1,w_2 \in \Omega_{n_0}^2} P(w_1, w_2) ~~ \text{applying Theorem~\ref{theorem:approx_trigonometric_polynomial}}\\
    &= (1 - \alpha)^{-1} \max_{1 \leq u,v \leq n_0} \norm{{D^\uparrow}^{(t)}_{u,v}}_F^{2} \ .
\end{align*}
\end{proof}


\begin{proof}




We can cast $\underset{1 \leq u,v \leq n}{\max} \norm{{D_{u,v}^\Uparrow}^{(t)}}_F^2$  as  the maximum over $\Omega_n^2$  of $\norm{E^{(t)}(w_1, w_2)}^{2}_F$. 
\begin{align*}
    \max_{1 \leq u,v \leq n} \norm{{D^\Uparrow}^{(t)}_{u,v}}_F^2 &= \max_{w_1,w_2 \in \Omega_n^2} \norm{E^{(t)}(w_1, w_2)}^{2}_F \\
    &\leq  \max_{w_1,w_2 \in [0, 2\pi]^2} \norm{E^{(t)}(w_1, w_2)}^{2}_F
\end{align*}
%
Using Lemma~\ref{lemma:inequality_btw_max_spectral_norm_density}:
\begin{align*}
    \max_{1 \leq u,v \leq n} \norm{{D^\Uparrow}^{(t)}_{u,v}}_F^2 \leq  (1 - \alpha)^{-1} \max_{1 \leq u,v \leq n_0} \norm{{D^\uparrow}^{(t)}_{u,v}}_F^2 \ .
\end{align*}
Finally using that 
\begin{align*}
 \sigma_1(C^\Uparrow) = \max_{1 \leq u,v \leq n}  \norm{{D^\Uparrow}_{u,v}}_2
 &\leq \max_{1 \leq u,v \leq n} \norm{{D^\Uparrow}^{(t)}_{u,v}}_F^{2^{1-t}} \ ,
\end{align*}
we have
\begin{equation*}
  \sigma_1(C^\Uparrow) \leq \left(\frac{1}{1 - \alpha}\right)^{2^{-t}}  \max_{1 \leq u,v \leq n_0} \norm{{D^\uparrow}^{(t)}_{u,v}}_F^{2^{1-t}} \ .
\end{equation*}
\end{proof}


\subsection{Proof of Theorem~\ref{thm:bound_circ_toep}}


\begin{proof}
%
Lemma~4 of \cite{yi2020asymptotic} shows that the spectral norm of the Toeplitz matrix is bounded by the spectral norm of its density matrix, defined in Eq.~(\ref{eq:spectral_density_matrix}): 
$\sigma_1(T^\Uparrow) \leq \sigma_1(E)$.
Then,
\begin{equation*}
 \sigma_1(T^\Uparrow)^{2^t} \leq \sigma_1(E)^{2^{t}} \leq \norm{E^{(t)}}_F^{2} \ .
\end{equation*}
Using Lemma~\ref{lemma:inequality_btw_max_spectral_norm_density} with $n$ sampling points:
\begin{align*}
      \max_{w_1,w_2 \in [0, 2\pi]} \norm{E^{(t)}(w_1, w_2)}_F^2 &\leq (1 - \alpha)^{-1} \max_{1 \leq u,v \leq n} \norm{{D^\Uparrow}^{(t)}_{u,v}}_F^{2} \ .
\end{align*}
Finally,
\begin{equation*}
     \sigma_1(T^\Uparrow) \leq \left(\frac{1}{1 - \alpha}\right)^{2^{-t}} \max_{1 \leq u,v \leq n} \norm{{D^\Uparrow}^{(t)}_{u,v}}_F^{2^{1-t}} \ .
\end{equation*}
\end{proof}


\subsection{Proof of Theorem~\ref{thm:improve_aol}}

    

\begin{proof}
Using Theorem~1 of \cite{araujo_unified_2022}, we have to show that $W^\top W - T \preceq 0$, for $t \geq 1$.
\\
The case $t=1$ comes from Theorem~3 of \cite{araujo_unified_2022} directly.
\\
For $t \geq 1$, using Theorem 3 of \cite{araujo_unified_2022} on $W^{(t)}$, 
\begin{align*}
 {W^{(t)}}^\top W^{(t)}  
 - \diag{\left( \sum_j \left| {W^{(t)}}^\top W^{(t)} \right|_{ij}
 \frac{q_j}{q_i}\right)} &\preceq 0 \\
 (W^\top W)^{2^{t}}  - \diag{\left( \sum_j \left| {W^{(t+1)}} \right|_{ij} 
 \frac{q_j}{q_i}\right)} &\preceq 0 \ .
\end{align*}
Using \cite[Theorem~7.9, p.~210]{zhang2011_matrix} we have that for $A, B \succeq 0$, $A^2 \preceq B^2 \implies A \preceq B$, 
\begin{align*}
 W^\top W  - \diag{\left( \sum_j \left| {W^{(t+1)}} \right|_{ij} \frac{q_j}{q_i} \right)}^{2^{-t}} &\preceq 0 \ .
\end{align*}
\end{proof}



\bibliographystyle{IEEEtran}
\bibliography{bibliography}

\begin{thebibliography}{10}
\providecommand{\url}[1]{#1}
\csname url@samestyle\endcsname
\providecommand{\newblock}{\relax}
\providecommand{\bibinfo}[2]{#2}
\providecommand{\BIBentrySTDinterwordspacing}{\spaceskip=0pt\relax}
\providecommand{\BIBentryALTinterwordstretchfactor}{4}
\providecommand{\BIBentryALTinterwordspacing}{\spaceskip=\fontdimen2\font plus
\BIBentryALTinterwordstretchfactor\fontdimen3\font minus
  \fontdimen4\font\relax}
\providecommand{\BIBforeignlanguage}[2]{{%
\expandafter\ifx\csname l@#1\endcsname\relax
\typeout{** WARNING: IEEEtran.bst: No hyphenation pattern has been}%
\typeout{** loaded for the language `#1'. Using the pattern for}%
\typeout{** the default language instead.}%
\else
\language=\csname l@#1\endcsname
\fi
#2}}
\providecommand{\BIBdecl}{\relax}
\BIBdecl

\bibitem{lecun_gradient_1998}
Y.~Lecun, L.~Bottou, Y.~Bengio, and P.~Haffner, ``Gradient-based learning
  applied to document recognition,'' \emph{Proceedings of the IEEE}, vol.~86,
  no.~11, pp. 2278--2324, 1998.

\bibitem{krizhevsky_imagenet_2012}
A.~Krizhevsky, I.~Sutskever, and G.~E. Hinton, ``{ImageNet} {Classification}
  with {Deep} {Convolutional} {Neural} {Networks},'' in \emph{Advances in
  {Neural} {Information} {Processing} {Systems}}, vol.~25, 2012.

\bibitem{tan_efficientnetv2_2021}
M.~Tan and Q.~Le, ``{EfficientNetV2}: {Smaller} {Models} and {Faster}
  {Training},'' in \emph{{International} {Conference} on {Machine} {Learning}},
  2021, pp. 10\,096--10\,106.

\bibitem{mallat_understanding_2016}
S.~Mallat, ``Understanding deep convolutional networks,'' \emph{Philosophical
  Transactions of the Royal Society A: Mathematical, Physical and Engineering
  Sciences}, vol. 374, no. 2065, p. 20150203, 2016.

\bibitem{li_survey_2022}
Z.~Li, F.~Liu, W.~Yang, S.~Peng, and J.~Zhou, ``A survey of convolutional
  neural networks: Analysis, applications, and prospects,'' \emph{IEEE Trans
  Neural Netw Learn Syst}, vol.~33, no.~12, pp. 6999--7019, 2022.

\bibitem{conneau2017very}
A.~Conneau, H.~Schwenk, Y.~Cun, and L.~Barrault, ``Very deep convolutional
  networks for text classification,'' in \emph{Long Papers - Continued}, ser.
  Conference of the European Chapter of the Association for Computational
  Linguistics, 2017, pp. 1107--1116.

\bibitem{gu_efficiently_2022}
A.~Gu, K.~Goel, and C.~Re, ``Efficiently modeling long sequences with
  structured state spaces,'' \emph{arXiv:2111.00396v3}, 2022.

\bibitem{Baevski20Wav2vec}
A.~Baevski, H.~Zhou, A.~Mohamed, and M.~Auli, ``wav2vec 2.0: A framework for
  self-supervised learning of speech representations,'' 2020.

\bibitem{zhang_graph_2019}
S.~Zhang, H.~Tong, J.~Xu, and R.~Maciejewski, ``Graph convolutional networks: a
  comprehensive review,'' \emph{Computational Social Networks}, vol.~6, 11
  2019.

\bibitem{pinson_linear_2023}
H.~Pinson, J.~Lenaerts, and V.~Ginis, ``Linear {CNNs} {Discover} the
  {Statistical} {Structure} of the {Dataset} {Using} {Only} the {Most}
  {Dominant} {Frequencies},'' in \emph{{International} {Conference} on
  {Machine} {Learning}}, 2023, pp. 27\,876--27\,906.

\bibitem{singh_hessian_2023}
S.~P. Singh, T.~Hofmann, and B.~Schölkopf, ``The {Hessian} perspective into
  the {Nature} of {Convolutional} {Neural} {Networks},'' in
  \emph{{International} {Conference} on {Machine} {Learning}}, 2023, pp.
  31\,930--31\,968.

\bibitem{tang_defects_2023}
L.~Tang, W.~Shen, Z.~Zhou, Y.~Chen, and Q.~Zhang, ``Defects of {Convolutional}
  {Decoder} {Networks} in {Frequency} {Representation},'' in
  \emph{{International} {Conference} on {Machine} {Learning}}, 2023, pp.
  33\,758--33\,791.

\bibitem{gavrikov_interplay_2023}
P.~Gavrikov and J.~Keuper, ``On the {Interplay} of {Convolutional} {Padding}
  and {Adversarial} {Robustness},'' \emph{arXiv:2308.06612}, 2023.

\bibitem{islam_how_2019}
M.~A. Islam, S.~Jia, and N.~D.~B. Bruce, ``How much {Position} {Information}
  {Do} {Convolutional} {Neural} {Networks} {Encode}?'' in \emph{International
  Conference on Learning Representations}, 2019.

\bibitem{semih_kayhan_translation_2020}
O.~Semih~Kayhan and J.~C. Van~Gemert, ``On {Translation} {Invariance} in
  {CNNs}: {Convolutional} {Layers} {Can} {Exploit} {Absolute} {Spatial}
  {Location},'' in \emph{{Computer} {Vision} and {Pattern} {Recognition}},
  2020, pp. 14\,262--14\,273.

\bibitem{zhang_making_2019}
R.~Zhang, ``Making {Convolutional} {Networks} {Shift}-{Invariant} {Again},'' in
  \emph{{International} {Conference} on {Machine} {Learning}}, 2019, pp.
  7324--7334.

\bibitem{bartlett_spectrally-normalized_2017}
P.~L. Bartlett, D.~J. Foster, and M.~J. Telgarsky, ``Spectrally-normalized
  margin bounds for neural networks,'' in \emph{Advances in {Neural}
  {Information} {Processing} {Systems}}, vol.~30, 2017.

\bibitem{miyato2018spectral}
T.~Miyato, T.~Kataoka, M.~Koyama, and Y.~Yoshida, ``Spectral normalization for
  generative adversarial networks,'' in \emph{International Conference on
  Learning Representations}, 2018.

\bibitem{scaman2018lipschitz}
A.~Virmaux and K.~Scaman, ``Lipschitz regularity of deep neural networks:
  analysis and efficient estimation,'' in \emph{Advances in Neural Information
  Processing Systems}, 2018.

\bibitem{cisse_parseval_2017}
M.~Cisse, P.~Bojanowski, E.~Grave, Y.~Dauphin, and N.~Usunier, ``Parseval
  {Networks}: {Improving} {Robustness} to {Adversarial} {Examples},'' in
  \emph{{International} {Conference} on {Machine} {Learning}}, 2017.

\bibitem{tsuzuku2018lipschitz}
Y.~Tsuzuku, I.~Sato, and M.~Sugiyama, ``Lipschitz-margin training: Scalable
  certification of perturbation invariance for deep neural networks,'' in
  \emph{Advances in Neural Information Processing Systems}, 2018.

\bibitem{delattre_efficient_2023}
B.~Delattre, Q.~Barth\'elemy, A.~Araujo, and A.~Allauzen, ``Efficient {Bound}
  of {Lipschitz} {Constant} for {Convolutional} {Layers} by {Gram}
  {Iteration},'' in \emph{{International} {Conference} on {Machine}
  {Learning}}, 2023.

\bibitem{golub2000eigenvalue}
G.~H. Golub \emph{et~al.}, ``Eigenvalue computation in the 20th century,''
  \emph{Journal of Computational and Applied Mathematics}, 2000.

\bibitem{wang_robust_2022}
W.~Wang, Z.~Dang, Y.~Hu, P.~Fua, and M.~Salzmann, ``Robust {Differentiable}
  {SVD},'' \emph{IEEE Trans Pattern Anal Mach Intell}, vol.~44, pp. 5472--5487,
  2022.

\bibitem{sedghi2019singular}
H.~Sedghi, V.~Gupta, and P.~Long, ``The singular values of convolutional
  layers,'' in \emph{International Conference on Learning Representations},
  2019.

\bibitem{wang2020orthogonal}
J.~Wang, Y.~Chen, R.~Chakraborty, and S.~X. Yu, ``Orthogonal convolutional
  neural networks,'' in \emph{Conference on Computer Vision and Pattern
  Recognition}, 2020.

\bibitem{el_mehdi_existence_2022}
E.~M. Achour, F.~Malgouyres, and F.~Mamalet, ``Existence, stability and
  scalability of orthogonal convolutional neural networks,'' \emph{Journal of
  Machine Learning Research}, vol.~23, no. 347, pp. 1--56, 2022.

\bibitem{farnia2019generalizable}
F.~Farnia, J.~Zhang, and D.~Tse, ``Generalizable adversarial training via
  spectral normalization,'' in \emph{International Conference on Learning
  Representations}, 2019.

\bibitem{araujo2021lipschitz}
A.~Araujo, B.~Negrevergne, Y.~Chevaleyre, and J.~Atif, ``On lipschitz
  regularization of convolutional layers using toeplitz matrix theory,''
  \emph{AAAI Conference on Artificial Intelligence}, 2021.

\bibitem{yi2020asymptotic}
X.~Yi, ``Asymptotic spectral representation of linear convolutional layers,''
  \emph{IEEE Trans Signal Process}, vol.~70, pp. 566--581, 2022.

\bibitem{singla2021fantastic}
S.~Singla \emph{et~al.}, ``Fantastic four: Differentiable and efficient bounds
  on singular values of convolution layers,'' in \emph{International Conference
  on Learning Representations}, 2021.

\bibitem{li2019preventing}
Q.~Li, S.~Haque, C.~Anil, J.~Lucas, R.~B. Grosse, and J.-H. Jacobsen,
  ``Preventing gradient attenuation in lipschitz constrained convolutional
  networks,'' in \emph{Advances in Neural Information Processing Systems},
  2019.

\bibitem{trockman2021orthogonalizing}
A.~Trockman \emph{et~al.}, ``Orthogonalizing convolutional layers with the
  cayley transform,'' in \emph{International Conference on Learning
  Representations}, 2021.

\bibitem{skew2021sahil}
S.~Singla and S.~Feizi, ``Skew orthogonal convolutions,'' in
  \emph{International Conference on Machine Learning}, 2021.

\bibitem{avidan_almost-orthogonal_2022}
B.~Prach and C.~H. Lampert, ``Almost-{Orthogonal} {Layers} for {Efficient}
  {General}-{Purpose} {Lipschitz} {Networks},'' in \emph{European Conference on
  Computer Vision}, 2022.

\bibitem{meunier_dynamical_2022}
L.~Meunier, B.~J. Delattre, A.~Araujo, and A.~Allauzen, ``A {Dynamical}
  {System} {Perspective} for {Lipschitz} {Neural} {Networks},'' in
  \emph{{International} {Conference} on {Machine} {Learning}}, 2022, pp.
  15\,484--15\,500.

\bibitem{araujo_unified_2022}
A.~Araujo, A.~J. Havens, B.~Delattre, A.~Allauzen, and B.~Hu, ``A {Unified}
  {Algebraic} {Perspective} on {Lipschitz} {Neural} {Networks},'' in
  \emph{International Conference on Learning Representations}, 2023.

\bibitem{fazlyab2019efficient}
M.~Fazlyab, A.~Robey, H.~Hassani, M.~Morari, and G.~Pappas, ``Efficient and
  accurate estimation of lipschitz constants for deep neural networks,'' in
  \emph{Advances in Neural Information Processing Systems}, 2019.

\bibitem{jain1989fundamentals}
A.~K. Jain, ``Fundamentals of digital image processing,'' in \emph{Fundamentals
  of digital image processing}.\hskip 1em plus 0.5em minus 0.4em\relax
  Englewood Cliffs, NJ: Prentice Hall, 1989.

\bibitem{henderson1981TheVM}
H.~V. Henderson and S.~R. Searle, ``The vec-permutation matrix, the vec
  operator and kronecker products: A review,'' \emph{Linear \& Multilinear
  Algebra}, vol.~9, pp. 271--288, 1981.

\bibitem{horn_matrix_2012}
R.~Horn and C.~Johnson, \emph{Matrix Analysis}.\hskip 1em plus 0.5em minus
  0.4em\relax Cambridge University Press, 2012.

\bibitem{friedland1991revisiting}
S.~Friedland, ``Revisiting matrix squaring,'' \emph{Linear algebra and its
  applications}, vol. 154, pp. 59--63, 1991.

\bibitem{song_why_2021}
Y.~Song, N.~Sebe, and W.~Wang, ``Why {Approximate} {Matrix} {Square} {Root}
  {Outperforms} {Accurate} {SVD} in {Global} {Covariance} {Pooling}?''
  \emph{arXiv:2105.02498}, 2021.

\bibitem{pfister2019bounding}
L.~Pfister and Y.~Bresler, ``Bounding multivariate trigonometric polynomials,''
  \emph{IEEE Trans Signal Process}, vol.~67, pp. 700--707, 2019.

\bibitem{Ryu19Plug}
E.~K. Ryu, J.~Liu, S.~Wang, X.~Chen, Z.~Wang, and W.~Yin, ``Plug-and-play
  methods provably converge with properly trained denoisers,'' in
  \emph{International Conference on Machine Learning}, 2019.

\bibitem{zhang2011_matrix}
F.~Zhang, \emph{Matrix Theory: Basic Results and Techniques}.\hskip 1em plus
  0.5em minus 0.4em\relax New York, NY: Springer, 2011.

\end{thebibliography}

\end{document}